\documentclass{article} 
\usepackage{iclr2026_conference,times}

\usepackage{amsmath,amsfonts,bm}

\def\eqref#1{equation~\ref{#1}}

\def\1{\bm{1}}

\DeclareMathAlphabet{\mathsfit}{\encodingdefault}{\sfdefault}{m}{sl}
\SetMathAlphabet{\mathsfit}{bold}{\encodingdefault}{\sfdefault}{bx}{n}

\newcommand{\E}{\mathbb{E}}

\newcommand{\R}{\mathbb{R}}

\DeclareMathOperator*{\bigcomp}{\bigcirc}

\usepackage{amsthm}
\newcommand{\norm}[1]{\left|\left|\,#1\,\right|\right|}
\newcommand{\mup}{$\mu$P\xspace}
\newcommand{\ep}{\varepsilon}

\newtheorem{lemma}{Lemma}

\usepackage{hyperref}
\usepackage{url}
\usepackage{booktabs}

\usepackage{multirow}
\usepackage{array}
\renewcommand{\arraystretch}{1.5} 

\usepackage{ulem}
\usepackage{xcolor, xspace}
\usepackage{tabularx}
\usepackage{graphicx}
\usepackage{adjustbox}
\usepackage{enumitem}
\usepackage{wrapfig}
\usepackage{subcaption}

\usepackage[utf8]{inputenc}
\DeclareUnicodeCharacter{03BC}{\ensuremath{\mu}}

\title{GQA-\mup: The Maximal Parameterization Update for Grouped Query Attention}

\author{
    Kyle~R.~Chickering\thanks{Equal contribution} \\
    UC Davis \& MBZUAI IFM
    \And
    Huijuan~Wang\footnotemark[1] \\
    USC \& MBZUAI IFM
    \And
    Mengxi~Wu\footnotemark[1] \\
    USC
    \And
    Alexander~Moreno \\
    MBZUAI IFM
    \And
    Muhao~Chen \\
    UC Davis
    \And
    Xuezhe~Ma \\
    USC \& MBZUAI IFM
    \And
    Daria Soboleva \\
    Cerebras
    \And
    Joel Hestness \\
    Cerebras
    \And
    Zhengzhong~Liu \\
    MBZUAI IFM
    \And
    Eric~Xing \\
    Carnegie Mellon University \& MBZUAI IFM
}

\iclrfinalcopy 
\begin{document}

\maketitle

\begin{abstract}
Hyperparameter transfer across model architectures dramatically reduces the amount of compute necessary for tuning large language models (LLMs). The maximal update parameterization ($\mu$P) ensures transfer through principled mathematical analysis but can be challenging to derive for new model architectures. 
Building on the spectral feature-learning view of \citet{yang2023spectral}, we make two advances.
First, we promote spectral norm conditions on the weights 
from a heuristic to the definition of feature learning, and as a consequence arrive at the Complete-P depth and weight-decay scalings without recourse to lazy-learning. Second, we consider a modified spectral norm 
that preserves 
the valid scaling law of 
network weights when 
weight matrices are not full rank. This 
enables (to our knowledge, the first) derivation of
\mup scalings for grouped-query attention (GQA). We demonstrate the efficacy of our theoretical derivations by showing learning rate transfer across the GQA repetition hyperparameter as well as experiments regarding transfer over weight decay.
\end{abstract}

\section{Introduction}
The maximal update parametrization (\mup) \citep{yangtensoriv, yangtensorv} provides principled rules for zero-shot learning rate transfer across model widths. 
Thus, large terminal model hyperparameters can be determined by sweeping a small proxy model. 
\mup has been used to train models up to at least 13B parameters with zero-shot transfer \citep{blake2023unit, dey2023cerebras, narayan2025mu}. 
Its applicability, however, has been largely limited to learning rate transfer across model widths. 
To broaden this scope, \citet{dey2025don} introduced Complete-P, 
extending the original prescriptions to weight decay and model depth. 
However, many common architectures that are widely deployed in production 
still lack established \mup scalings.

This paper seeks to close this gap by extending the spectral \mup framework of \cite{yang2023spectral} to be more practically useful in deriving \mup prescriptions for novel architectures. As an example of the utility of our framework, we derive (to our knowledge, the first) \mup scaling for grouped-query attention (GQA) \citep{ainslie2023gqa}. Our analysis reveals that GQA surfaces several difficulties that prior work has left unaddressed. First, when using GQA the original \mup implementation passes coordinate checks, 
i.e., the customary correctness tests for the implementation.
However, empirical analysis shows that the original \mup implementation 
fails to transfer learning rates, seemingly contradicting established theory (see Figures~\ref{fig:gqa_ablation_no_fsdp} and ~\ref{fig:dw_adam}). 
We resolve this by extending the spectral-norm version of \mup introduced in \cite{yang2023spectral}, and showing that the original \mup implementation does not pass a more rigorous spectral-norm coordinate check. Second, the 
intrinsic low rank of GQA weight matrices 
skews the expected size of layer outputs. 
To address this issue, we introduce 
a new norm, namely the expected operator norm, to replace the spectral norm in spectral \mup theory and restore the desired scaling behavior.

Our primary contributions are threefold:
\begin{enumerate}
    \item We extend the spectral \mup theory of \citet{yang2023spectral}, which allows derivations of \mup for more advanced architectures like weight decay, recursion blocks, and GQA. Our work provides, 
    to our knowledge, the first derivation of \mup scaling for GQA.
    \item We perform empirical analysis to validate the theory and offer 
    practical guidance for 
    learning rate transfer across GQA settings. In particular, we suggest that 
    transferring across different numbers of GQA repetitions leads to noisy transfer dynamics, suggesting caution when attempting to transfer learning rate.
    \item Our experiments show that, 
    with the correct scalings, both weight decay and the training-time constant $\tau_{\text{epoch}}$ introduced in \cite{wang2024set} appear to be transferable.
\end{enumerate}

\begin{figure*}[!t]
    \centering
    \includegraphics[width=\linewidth]{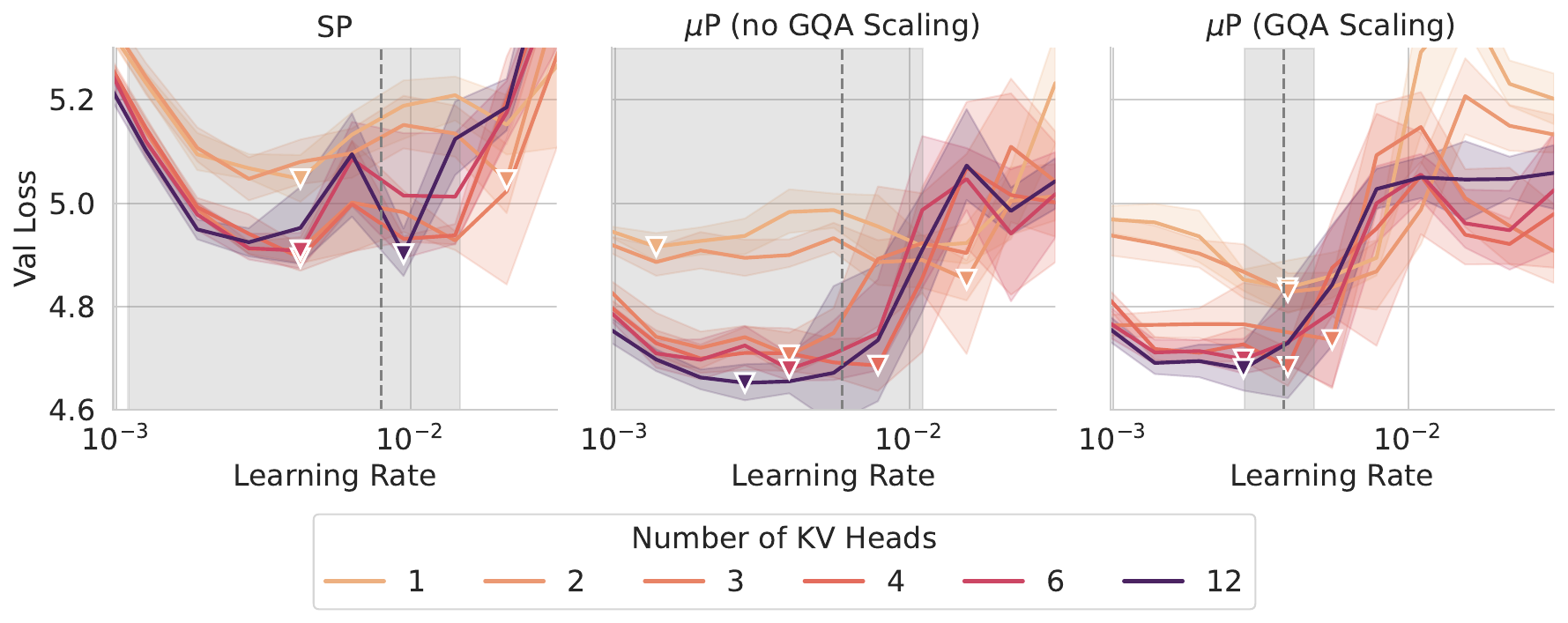}
    \caption{Comparison of the standard parameterization (left), the vanilla Adam-\mup parameterization (middle), and our GQA-\mup scaling (right).
    For a fixed model size, we vary the number of KV heads.
    The dashed lines indicate the mean optimal learning rates for each parameterization, and the shaded grey region denotes the standard deviation of the optimal learning rates.
    All models are trained to 10 tokens per parameter (TPP).
    Additional details can be found in Appendix~\ref{app:gqa_exp}.}
    \label{fig:gqa_ablation_no_fsdp}
\end{figure*}

\section{Related Work}
\textbf{Foundations of \mup:} \mup builds on a series of works by Yang, developing the Tensor Programs framework~\citep{yang2019tensori, yang2020tensorii, yang2020tensoriii, yangtensoriv, yangtensorv, yangtensorvi}. This line of work uses random matrix theory to carefully analyze the mathematical properties of neural networks during training, while also empirically demonstrating that these theoretical approaches remain valuable for real-world deep learning. Within the framework of Tensor Programs, \cite{yangtensorv} derives the well-known \mup scaling laws for width under SGD and Adam training. The final paper in the series~\cite{yangtensorvi} attempts to extend \mup to depth scalings. However, they were unable to extend to the case of residual blocks with standard configurations for the hidden layers. Finally, the foundation of the mathematical framework presented in this work builds on~\cite{yang2023spectral}, who show an alternative derivation of the results in~\cite{yangtensorv} based on spectral norms.

\textbf{Models using GQA:} Group-query attention (GQA)~\citep{ainslie2023gqa} is an efficient attention mechanism that reduces memory usage by sharing key and value heads across groups of query heads. 
Due to its favorable trade-off between memory efficiency and model performance, GQA has been widely used in modern large language models, including Mistral 7B \citep{Jiang2023Mistral7}, LLaMA 3 \citep{grattafiori2024llama},  Qwen3~\citep{yang2025qwen3} and K2-V2~\citep{liu2025k2}. 

\textbf{Extensions of \mup:} The original \mup formulation presented in \cite{yangtensorv} applies only to scaling the width of a fixed depth, fixed batch size neural network. While already a powerful tool, later authors have sought to extend the principles of \mup to cover cases not covered by the original formulation. \cite{dey2023cerebras} do large-scale validation experiments using \mup and find empirical evidence that learning rate can transfer across batch and dataset size. \cite{dey2023cerebras} suggests \mup-type scalings for weight decay, the Adam $\varepsilon$, and depth. Their contributions to depth scaling are most notable, as their empirical findings contradict the scaling presented in~\cite {yangtensorvi}. However, their extensive empirical analysis suggests that the scaling they derive is correct. We arrive at the same scaling in Section~\ref{sec:depth_scaling} using the framework we outline in this paper. More recently~\cite{mlodozeniec2025completed} extended the work of Dey et al. \cite{dey2025don} by extending the SDE parameterization to cover hyperparameter transfer for batch size, as well as showing the value of per-layer learning rate tuning.

\cite{blake2023unit} apply \mup in the context of large-scale, low-precision LLM training. They use ABC parameterizations to apply the \mup scaling rules while maintaining unit variance for all layers in the network, which they refer to as unit scaling-\mup. Additionally, they empirically validate that learning rate transfer persists across datasets, batch sizes, depths, and training iterations under controlled conditions. \cite{narayan2025mu} suggest a different, more simplified version of the unit scaling-\mup which they also show works for training low-precision networks with \mup.

Finally, a related subsequent work \cite{zheng2026spectral} has proposed a similar theoretical framework to ours. Both our work and theirs share the perspective that the spectral norm provides a principled alternative to Tensor Programs \cite{yangtensoriv} for deriving \mup. However our works differ in scope and motivation. \cite{zheng2026spectral} systematically apply their framework to a broad class of optimizers under width and depth scaling, while our work identifies the expected operator norm as necessary to correctly handle rank-degenerate weights and using this norm to provide the first derivation of \mup for GQA.




\section{Deriving Novel Maximal Update Parameterizations}\label{sec:derivation}










Consider a collection of weight matrices $\bm{W}^\ell \in \R^{n_\ell\times m_\ell}$ in a neural network, indexed by layer $\ell$. \cite{yang2023spectral} proves that conditions imposed upon the weight matrices of a network imply feature learning (and thus learning rate transfer) as defined in \cite{yangtensorv} (see Equation~\ref{eq:yang_feature}). 
For initial weights $\bm{W}_0^\ell$ and iterates $\bm{W}_t^\ell = \bm{W}_0^\ell + \sum_{k=1}^t\Delta \bm{W}_t^\ell$, where $\Delta \bm{W}_t^\ell = \bm{W}_t^\ell - \bm{W}_{t-1}^\ell$, \cite{yang2023spectral} suggests that both the initialization and the updates must satisfy: 
\begin{align}\label{eq:spectral_condition}
    &\norm{\bm{W}^\ell_0}=\Theta(\sqrt{n_\ell}/\sqrt{m_\ell}), \quad \norm{\Delta \bm{W}_t^\ell}=\Theta(\sqrt{n_\ell}/\sqrt{m_\ell}),
\end{align}
where $\norm{\bm{W}}:=\sup_{\norm{x}_2=1}\norm{\bm{W}x}_2$ is the usual spectral (or induced) norm. This spectral perspective on feature learning is powerful, and we introduce three minor but important modifications that enable us to extend the method of \cite{yang2023spectral} to cover novel architectures like GQA. 

\begin{figure}[!t]
    \centering
    \footnotesize
    \includegraphics[width=0.65\linewidth]{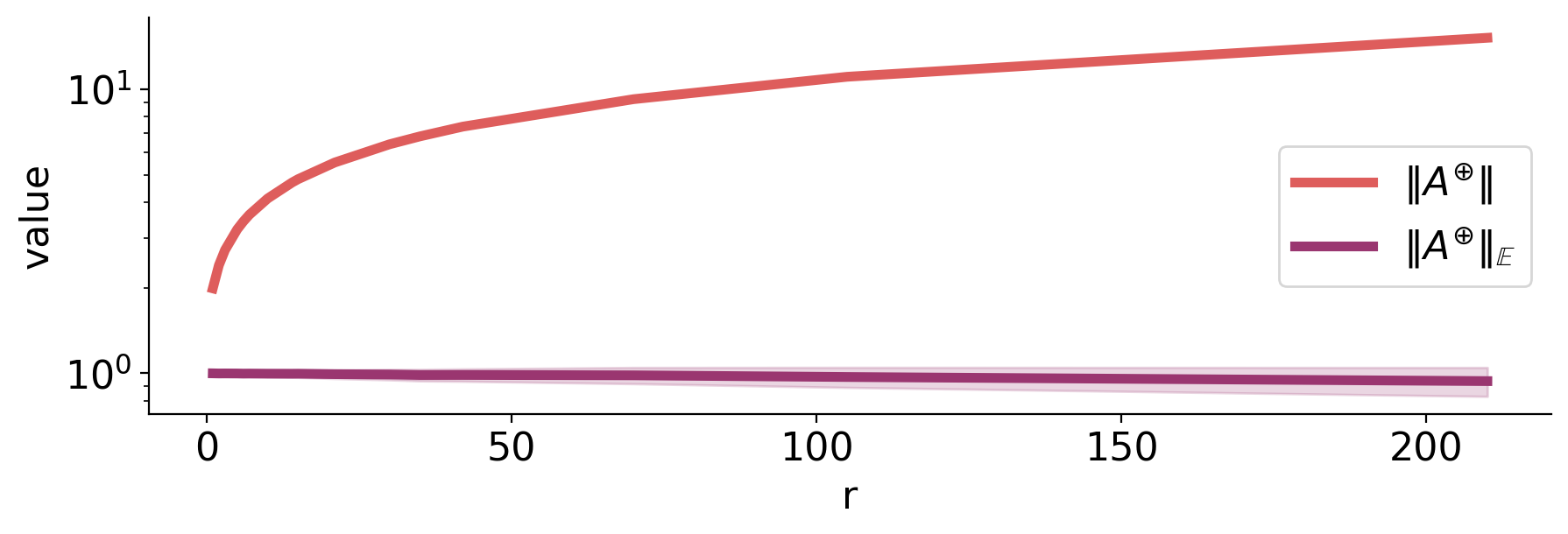}
    \caption{Demonstration of the failure of the spectral norm to accurately capture the behavior for low-rank matrices when the inputs are randomly sampled i.i.d.\ from $\mathcal{N}(0, 1)$.
    $r$ is the number of key-value head repetitions and $r=1$ corresponds to the setting without GQA.
    Each point is averaged over $1000$ independent draws of $\bm{A}$, with the shaded band showing $\pm 1$ standard deviation. 
    }
    \label{fig:expected_norms}
\end{figure}


\textbf{Analysis Under a New Norm:} The spectral norm can be interpreted as the maximal deformation of an input vector induced by an operator $\varphi: \R^m\rightarrow \R^n$. For full-rank operators, such as dense feed-forward layers, random matrix theory shows that the quantitative value of the spectral norm is attained asymptotically. 
In the classical case of an $n\times n$ random matrix $A$, we have the sharp asymptotic relation $\norm{A}=2\sqrt{n}$ as $n\rightarrow \infty$.

However, for rank-degenerate matrices like those used in GQA, the spectral norm is not attained asymptotically in practice. 
The reason is that, as shown by Tensor Programs~\cite{yangtensoriv}, the inputs to a GQA layer during training are i.i.d., and therefore, for rank-degenerate matrices, the vectors that cause this ``maximal deformation" occur with probability zero! A visualization of this discrepancy can be seen in Figure~\ref{fig:expected_norms}.
Instead, we should use a notion of size that reflects the actual deformation encountered during training.

To this end, let $\Omega$ be the probability distribution of the input vectors. We define the \textbf{expectation operator norm} as\footnote{Technically, the object we define as $\norm{A}_{\E, \Omega, p}$ is only a seminorm without further constraints on $\Omega$. In particular, if $\text{supp}\,\Omega\ne\R^n$ then it is possible for all random vectors $x\sim \Omega$ to lie in the nullspace of $A$. This edge case does not occur in neural network training.} 
\begin{align}\label{eq:expected_operator_norm}
    \norm{\bm{A}}_{\E, \Omega, p}:=\E_{x\sim \Omega}\left[\,\frac{\norm{\bm{A}x}_p}{\norm{x}_p}\,\right].
\end{align}
Throughout this paper, we adopt the convention $\norm{\bm{A}}_{E}=\norm{\bm{A}}_{\E, \mathcal{N}(0, 1), 2}$, where $x\sim \mathcal{N}(0, 1)$ has i.i.d. entries. Crucially, when $A$ is square with i.i.d. entries, it has full rank with probability one, and we obtain the asymptotic relationship $\norm{\bm{A}}_{\E}=\Theta(\norm{\bm{A}})$. A proof is provided in Lemma~\ref{lem:asymptotc_equivalence_of_norms}.

\textbf{Operator-Norm Focused Feature Learning:}
\cite{yang2023spectral} shows that constraining the spectral norm of the weight matrices implies feature learning in the sense of \cite{yangtensoriv}, where feature learning is defined to occur when
\begin{align}\label{eq:yang_feature}
    \norm{h^\ell_0}_2 = \Theta(\sqrt{n}), \qquad \norm{\Delta h^\ell_t}_2=\Theta(\sqrt{n}),
\end{align}
for all pre-activations $h^\ell$. In particular, \cite{yang2023spectral} prove that enforcing condition~\eqref{eq:spectral_condition} on spectral norms implies~\eqref{eq:yang_feature}. However, the converse does not hold: feature learning in the sense of~\eqref{eq:yang_feature} may still occur even if the weight matrices do not scale according to~\eqref{eq:spectral_condition}.

Consider a hidden layer $h(x) = \bm{W}x$ with trainable weights $\bm{W}\in \R^{n\times n}$ and an additional scaling parameter, the number of layers $L > 1$ independent of $n$, for which we want to ensure feature learning as $L\rightarrow \infty$. Under proper initialization, i.e. $\norm{\bm{W}_0}=\Theta(1)$, we have $\norm{h(x)}_2 = \Theta(\sqrt{n})$ for $x\in \R^n$ with $\norm{x}_2 = \Theta(\sqrt{n})$, as required by feature learning. Now suppose the learning rate is set incorrectly, and $\norm{\Delta \bm{W}_t} = \Theta(L^{-\alpha})$ for some $0 < \alpha$. Then the weight update takes the form
\begin{align*}
    \Delta h_t = \bm{W}_tx_t - \bm{W}_{t-1}x_{t-1} = \Delta \bm{W}_tx_t + \bm{W}_{t-1}\Delta x_{t}.
\end{align*}
Assuming that these terms do not exactly cancel, and noting that $\norm{\Delta x_t}_2=\Theta(\sqrt{n})$, we have that
\begin{align*}
    \norm{\Delta h_t} = \Theta(\sqrt{n}(1 + L^{-\alpha})) = \Theta(\sqrt{n}), 
\end{align*}
and thus this layer satisfies feature learning in the sense of~\eqref{eq:yang_feature}. For GQA, this precise situation arises, and the subtle failure of the terms $\Delta h_t$ to properly scale leads to a failure of learning rate transfer (see Figures~\ref{fig:dw_adam} and~\ref{fig:gqa_ablation_no_fsdp} below).

This analysis shows that the spectral condition of~\eqref{eq:spectral_condition} is a stronger notion of feature learning than~\eqref{eq:yang_feature} and we propose using it as the \textbf{definition} of feature learning. This perspective has beneficial practical consequences. When doing coordinate checking to validate a \mup implementation as shown in \citep{yangtensorv}, we found that directly analyzing the weight matrices proves more effective than analyzing only the activations (see Figure~\ref{fig:df_input_norm}). This point is discussed further below.

\textbf{A Functional Analytic View of Layer-Wise Computation:} Modern machine learning architectures consist of more than dense feed-forward units. Thus, we propose focusing on the computational units of the network rather than specifically focusing on matrices. In the case of dense feed-forward layers, these notions coincide. But for residual layers our perspective offers a more unifying approach. Concretely, we regard a neural network not only as a compositional sequence of matrix multiplications, but as a compositional sequence of abstract, generally non-linear mappings $\varphi^\ell:\R^m\rightarrow \R^n$. 

We suggest that the first part of the spectral condition in~\eqref{eq:spectral_condition} should be applied to each compositional unit, rather than the matrices themselves. Starting with the end-to-end computation of the network, we recursively apply this condition to all mappings $\varphi^\ell$. In conjunction with requiring that all trainable parameters satisfy both parts of~\eqref{eq:spectral_condition}, this leads to a unified treatment of residual layers which we discuss  below.

\begin{table*}[t]
\centering
\footnotesize
\caption{The table summarizes the parameterization of Transformers with Grouped-Query Attention (GQA), where $n$ denotes the input dimension and $r$ is the number of key-value head repetitions. Modifications specific to GQA are highlighted in blue. The derivations of learning rate and weight decay follow the AdamW implementation in PyTorch.}
\label{tab:mup-gqa}
\renewcommand{\arraystretch}{1.3}
\begin{tabular}{cccccc}
\toprule
 & Embed. & Unemb. & Attn. (Q, O) & Attn. (K, V) & Feed-forward \\
\toprule
Init. Var.
& \textcolor{gray}{1}
& \textcolor{gray}{1}
& \textcolor{gray}{$1/\sqrt{\text{n}}$}
& \textcolor{gray}{$1/\sqrt{n}$}
& \textcolor{gray}{$1/\sqrt{\text{n}}$} \\
Multiplier
& \textcolor{gray}{1}
& \textcolor{gray}{$1/\text{n}$}
& \textcolor{gray}{1}
& \textcolor{gray}{1}
& \textcolor{gray}{1} \\
LR
& \textcolor{gray}{1}
& \textcolor{gray}{1}
& \textcolor{gray}{$1/n$}
& \textcolor{blue}{${(1 + \sqrt{\text{r}})}/{(2n)}$}
& \textcolor{gray}{$1/n$} \\
Weight Decay
& \textcolor{gray}{1}
& \textcolor{gray}{1}
& \textcolor{gray}{$n$}
& \textcolor{blue}{${2n}/{(1 + \sqrt{\text{r}})}$}
& \textcolor{gray}{n} \\
\bottomrule
\end{tabular}
\vspace{0.1cm}
\end{table*}

\subsection{Weight Decay}

Weight decay is commonly applied in deep learning to stabilize model training dynamics~\citep{loshchilov2017decoupled, andriushchenko2023we}. 
For concreteness, we focus on AdamW \citep{loshchilov2017decoupled} in this section, although our framework extends well to other optimizers, such as MuON~\citep {jordan2024muon}, with weight decay. AdamW modifies the Adam weight update~\eqref{eq:adam_update} by including a weight decay term with the associated weight decay hyperparameter\footnote{
    We focus on \textbf{coupled} weight decay, which is the type of weight decay included in \texttt{PyTorch} \citep{adamw-decoupling-blog}. However, the weight decay introduced in \cite{loshchilov2017decoupled} is \textbf{decoupled} and given by $\Delta\bm{W}_t = -\lambda\,\bm{W}_t - \eta\hat{\bm{r}}_t$. Our results still apply in this case and prescribe the scaling $\lambda=\lambda_0$, where $\lambda_0$ is the base model weight decay, to ensure the terms all have the same size in norm. In other words, when using decoupled weight decay, the base weight decay term should not scale with model size. See also \cite{dey2025don}.}

We begin by defining the update rule for AdamW. The total parameter update $\Delta\bm{W}_t$ is composed of two distinct terms: a weight decay term and the standard Adam gradient-based update $\hat{\bm{r}}_t$:
\begin{align}\label{eq:adamw_update}
    \Delta\bm{W}_t = -\lambda\eta\,\bm{W}_t - \eta\hat{\bm{r}}_t.
\end{align}
For the learning dynamics to remain consistent (transferable) as we scale the network width $n$, the spectral norms of these two contributions must scale identically. Specifically, the weight decay magnitude must match the gradient update magnitude, and both must remain stable relative to the weights themselves. Mathematically, this balance condition is expressed as:
\begin{align*}
    \norm{\Delta \bm{W}_t} = \Theta(\lambda\eta\norm{\bm{W}_t})=\Theta(\eta\norm{\bm{r}_t}) = \Theta(1).
\end{align*}
\looseness=-1
Given that $\norm{\Delta \bm{W}_t} \approx \norm{\bm{W}_t}$ in the feature learning limit, it follows that the weight decay coefficient must satisfy $\lambda\eta=\Theta(1)$. Recall that \mup prescribes specific learning rates $\eta$ depending on the layer type: $\eta = \Theta(1)$ for input layers, and $\eta=\Theta(1/n)$ for hidden and output layers. To maintain the balance described above, the base weight decay $\lambda_0$ must be scaled as follows. For input Layers, since $\eta = \Theta(1)$, we require $\lambda^0=\Theta(1)$. For the hidden and output Layers, since $\eta = \Theta(1/n)$, we require $\lambda^\ell = \Theta(n)$.

Let us analyze the dynamics of a hidden layer where we introduce a scaling error $\delta > 0$. Suppose we set $\lambda^\ell = \lambda_0 n^{1+\delta}$. To maintain a stable update size, the learning rate needs to compensate, scaling as $\eta = \Theta(n^{-1-\delta})$. As $n \rightarrow \infty$, 
$\eta\hat{\bm{r}}_t$ becomes negligible. The weight decay term dominates
\(
\Delta\bm{W}_t \approx -\eta_0\lambda_0\,\bm{W}_t.
\) Consequently, the model ignores the data entirely, and the weights simply decay toward $\bm{0}$ without learning features. Conversely, suppose we set $\lambda = \lambda_0 n^{1-\delta}$. This implies standard \mup learning rate $\eta=\Theta(n^{-1})$ is relatively larger compared to the decay strength. As $n \rightarrow \infty$, 
\(
\Delta \bm{W}_t \approx - \eta\hat{\bm{r}}_t
\)
vanishes. The weight decay is effectively ignored. The algorithm collapses back to standard Adam, losing the regularization benefits of AdamW. Thus, the scaling $\lambda = \Theta(n)$ maintains the necessary equilibrium between regularization and feature learning at large widths.

Recent studies by \citet{wang2024set} and \citet{dey2025don} have derived a similar relation between learning rate and weight decay through different methods. We experimentally validate this relationship in Figure~\ref{fig:voronoi_tau}.


\subsection{Grouped Query Attention}\label{sec:gqa}
Grouped query attention (GQA) reduces computational cost by repeating the key and value heads in the Transformer \citep{ainslie2023gqa}. In a standard multi-headed attention layer, the key and value projections are given by weights $\bm{W}_K \in \R^{n\times n}$ and $\bm{W}_V \in \bm{R}^{n\times n}$, where $n$ is the embedding dimension. These matrices are partitioned into $H$ heads of size $n / H$ each and the $i$-th head is computed as $k_i = (\bm{W}_Kx)_{i}, \, v_i = (\bm{W}_Vx)_i$.
In GQA, the number of parameters is reduced by using only $p$ distinct key/value heads, where $H / p=r$ denotes the number of repetitions of each key/value head group.
We then define matrices $\bm{W}_{p, K}, \bm{W}_{p, V} \in \R^{\frac{n}{r}\times n}$, and construct the full key and value weights by concatenating along the output dimension
\begin{align}
    &\bm{W}_K^\oplus=\bigoplus_{m=1}^r\bm{W}_{p, K}, \quad \bm{W}_V^\oplus=\bigoplus_{m=1}^r\bm{W}_{p, V},
\end{align}
where $\oplus$ denotes concatenation along the first dimension\footnote{Note that concatenation and matrix multiplication commute: if $\bm{A} \in \R^{m\times n}$ and $x\in \R^n$, we have $\bm{A}^\oplus x = (\bm{A}x)^\oplus$, which follows directly by writing the product in it's index form.}.

Consider the initial weight matrix $\bm{W}_0$ for either the key or value projections, and its concatenation version $\bm{W}^\oplus_0$, and let $\bm{W}_t$ and $\bm{W}_t^\oplus$ denote their corresponding weight updates. To begin, applying the law of large numbers and the central limit theorem to~\eqref{eq:expected_operator_norm}, we obtain

\[
\resizebox{\linewidth}{!}{$%
\begin{aligned}
    \norm{\bm{W}^{\oplus}_0}_{\E} 
    = \E_x\left[\,\Theta\left(\frac{\left(\sum\limits_{k=1}^n\left(\sum\limits_{j=1}^n(W_0^\oplus)_{kj} x_j\right)^2\right)^{\frac{1}{2}}}{(\sum\limits_{k=1}^nx_j^2)^{\frac{1}{2}}}\right)\,\right] = \E_x\left[\,\Theta\left(\frac{\left(r\sum\limits_{k=1}^{\frac{n}{r}}\left(\sum\limits_{j=1}^n(W_0)_{kj} x_j\right)^2\right)^{\frac{1}{2}}}{(\sum\limits_{k=1}^nx_j^2)^{\frac{1}{2}}}\right)\,\right].
\end{aligned}
$}
\]
Therefore, 
\[  \norm{\bm{W}^{\oplus}_0}_{\E}  = \Theta\left(\frac{(r\times \frac{n}{r}\times n\times \sigma^2)^{\frac{1}{2}}}{n^{\frac{1}{2}}}\right) = \Theta\left(\sigma n^{\frac{1}{2}}\right).\]

\looseness=-1
To satisfy the spectral condition in~\eqref{eq:spectral_condition}, we require $\sigma = \Theta(n^{-1/2})$. Importantly, this corresponds to the expected operator norm for $\bm{W}^\oplus$, not the spectral norm of the constituent matrix $\bm{W}$. Because $\bm{W}_0$ has full rank with probability 1, its spectral norm can be computed directly using Bai-Yin \citep{bai1993limit, yin1988limit}
\begin{align}\label{eq:gqa_deriv}
    \norm{\bm{W}_0} = \Theta\left(\sigma\left(\sqrt{n} + \frac{\sqrt{n}}{\sqrt{r}}\right)\right) = \Theta\left(\frac{1 + \sqrt{r}}{\sqrt{r}}\right).
\end{align}
Moreover, in terms of spectral norms, $\norm{\bm{W}_0^\oplus}=\sqrt{r}\norm{\bm{W}_0}$ (Lemma~\ref{lem:concat}), so that the spectral norm and the expected operator norm do not agree in this setting (see Figure~\ref{fig:expected_norms}).



The computation in~\eqref{eq:gqa_deriv} is critical for determining the required learning rate, since we require $\norm{\Delta \bm{W}_t}=\Theta(\norm{\bm{W}_0})$. To this end, we compute $\eta$ in the usual manner. Assuming the use of the Adam optimizer with update step $\hat{\bm{r}}_t$, we have

\begin{align*}
    \norm{\Delta\bm{W}_t} = \eta\norm{\hat{\bm{r}}_t} = \Theta\left(\frac{\eta n}{\sqrt{r}}\right) = \Theta\left(\frac{1 + \sqrt{r}}{\sqrt{r}}\right).
\end{align*}
From this we easily deduce that $\eta = \Theta\left(\frac{1+\sqrt{r}}{n}\right)$. We normalize by a factor of two to ensure that when $r=1$ our scalings agree with the usual full-rank hidden layer scalings:
\begin{align}
    \sigma = \frac{1}{\sqrt{n}}\sigma_0, \qquad \eta =\frac{1+\sqrt{r}}{2n}\eta_0.
    \label{eq:learning_rate_scale}
\end{align}

Through the above derivations, we arrive at the parameterization of Transformers with GQA as summarized in Table~\ref{tab:mup-gqa}.

Our GQA-$\mu$P can adapt to any group size by following equation~\ref{eq:learning_rate_scale}. Empirically, we evaluate $r \in \{1,2,3,4,5,6,12\}$. These values cover commonly used settings, including $r=4$ for Llama-3 8B and Mistral 7B, $r=8$ for Llama-2/3 70B and Qwen-2.5 72B, and $r=12$ for Cohere Command R+ (e.g., 96 query heads and 8 key/value heads).

To clarify, our mathematical framework is asymptotically valid with respect to $r$. However, in practice, $r$ is not scaled to infinity in the same way as the network width $n$ is. Since $n \geq r$, the extreme limit where $r = n$ would force the model to have single-parameter attention heads. Thus, while our asymptotic analysis is sound, this particular infinite limit is not a realistic scenario that would actually be used. The typical range of $r$ in practice usually does not exceed 16, which may not be asymptotically
large. However, we would like to emphasize that our primary objective in scaling with $r$ is empirical: to prevent learning-rate drift when changing the number of repetitions $r$. Our derivations remain valid in the asymptotic limit, which is merely a mathematical consequence of the theory rather than the practical setting we target.

\subsection{Complete-P Depth Scaling}\label{sec:depth_scaling}

We now turn to the depth scaling of residual networks. Complete-P~\citep{dey2025don}, derives depth scalings by relying on an additional desideratum of ``no lazy learning.'' We show that we do not need this additional assumption. We demonstrate that applying the standard $\mu$P desideratum to the spectral norm automatically prevents the no lazy learning.'' assumption. In this section, we show that applying our framework naturally recovers the exact same scaling as in Complete-P.

Consider the stacked hidden layers, where the output of the $\ell$-th layer is given by the residual update:
\begin{equation}
    G^\ell(x) = x + \beta g^\ell(x) \quad \text{for } 1 \leq \ell \leq L.
\end{equation}
Here, $x \in \R^n$ is the input, $g^\ell: \R^n \rightarrow \R^n$ represents the residual branch, and $\beta$ is a scaling constant independent of the layer index $\ell$. To formalize our analysis and address the network dynamics rigorously in the large-width ($n \to \infty$) and large-depth ($L \to \infty$) limits, we state the following assumptions explicitly:
\begin{itemize}[itemsep=2pt, parsep=0pt, topsep=0pt, partopsep=0pt, leftmargin=*]
\item \textbf{Assumption 1 (Input Scaling):} The network inputs satisfy $\norm{x}_2 = \Theta(\sqrt{n})$. This ensures that the base compositional units satisfy the spectral condition in \eqref{eq:spectral_condition}.    \item \textbf{Assumption 2 (Stable Operator Norms):} $\norm{g^\ell} = \Theta(1)$, $\norm{\Delta g^\ell} = \Theta(1)$, $g^{\ell}$ is full-rank, and $\beta\norm{g^{\ell}} < 1$.
    \item \textbf{Assumption 3 (No Exact Cancellation):} This is a standard assumption in the \mup{} and Tensor Programs literature~\citep{yangtensoriv}. For more intuition about this, please refer to ~\ref{sec:intuition_cancellation}.
\end{itemize}

Let $\overline{G}_t^\ell = \bigcomp_{k=1}^\ell G_t^k$ represent the composition of the first $\ell$ layers at training step $t$. We can express the network recursively:
\begin{equation}
    \overline{G}_t^\ell = \overline{G}_t^{\ell - 1} + \beta g_t^\ell \circ \overline{G}_t^{\ell-1}.
\end{equation}

Under Assumption 3, taking the norm yields the following asymptotic relation:
\begin{equation}
    \norm{\overline{G}_t^\ell} = \Theta\left(\norm{\overline{G}_t^{\ell-1}} + \beta \norm{g_t^\ell} \norm{\overline{G}_t^{\ell-1}}\right).
\end{equation}

Because Assumption 2 dictates $\norm{g_t^\ell} = \Theta(1)$, every layer effectively adds a proportional factor of $\beta$ to the norm estimate. We can formalize the resulting depth-dependent bound via a recursive induction argument. For the base case, at the input, $\norm{\overline{G}_t^0} = \norm{I} = 1 \leq \Theta(1)$. Assume the bound holds for layer $\ell$, such that $\norm{\overline{G}_t^\ell} \leq \Theta(1 + \ell\beta)$. For layer $\ell + 1$, we substitute the hypothesis into our asymptotic relation:
    \begin{align*}
        \norm{\overline{G}_t^{\ell+1}} \leq \Theta\left(\norm{\overline{G}_t^\ell} + \beta \norm{\overline{G}_t^\ell}\right) \leq \Theta\left((1 + \ell\beta) + \beta(1 + \ell\beta)\right) = \Theta\left(1 + \ell\beta + \beta + \ell\beta^2\right).
    \end{align*}
    Because we are analyzing regimes where $\beta < 1$, the higher-order term $\ell\beta^2$ is dominated by the linear terms ($\beta^2 < \beta$). Thus, the expression simplifies to:
    \begin{equation*}
        \norm{\overline{G}_t^{\ell+1}} \leq \Theta(1 + (\ell + 1)\beta).
    \end{equation*}
By induction, evaluating this at the final layer yields the total forward pass bound:
\begin{equation}
    \norm{\overline{G}_t^L} = \Theta(1 + L\beta).
\end{equation}

To maintain stable representations and avoid network explosion or vanishing signals, we require the final norm to be independent of depth, i.e., $\Theta(1 + (\ell + 1)\beta) = \Theta(1)$. Setting this equality forces $L\beta = \Theta(1)$, which directly implies that $\beta = \Theta(L^{-1})$.

\looseness=-1
We single out a minor point of confusion: for a fixed, finite depth network (e.g., $L=2$), a constant value such as $\beta=1/2$ naturally satisfies this requirement because $1/2 = O(1/L)$ for that specific $L$. However, when considering the asymptotic limit where $L \to \infty$, the parameter $\beta$ cannot be a static constant larger than $1/L$; it must shrink in exact proportion to the total depth to prevent the $L\beta$ term from diverging. Conversely, choosing a faster shrinking exponent, such as $\beta = \Theta(L^{-\alpha})$ for $\alpha > 1$ \citep{yangtensorvi}, causes the residual branches to vanish entirely, degenerating into trivial dynamics. Therefore, $\beta = \Theta(L^{-1})$ guarantees stability without sacrificing expressivity, arriving at the exact same bound as \cite{dey2025don} through an alternate derivation.

\section{Empirical Results}\label{sec:empirical}


In this section, we present our empirical results. Details of model configurations and experimental setups are provided in Appendix~\ref{app:model_configs}.




\begin{figure*}[!t]
    \centering
    \includegraphics[width=0.84\linewidth]{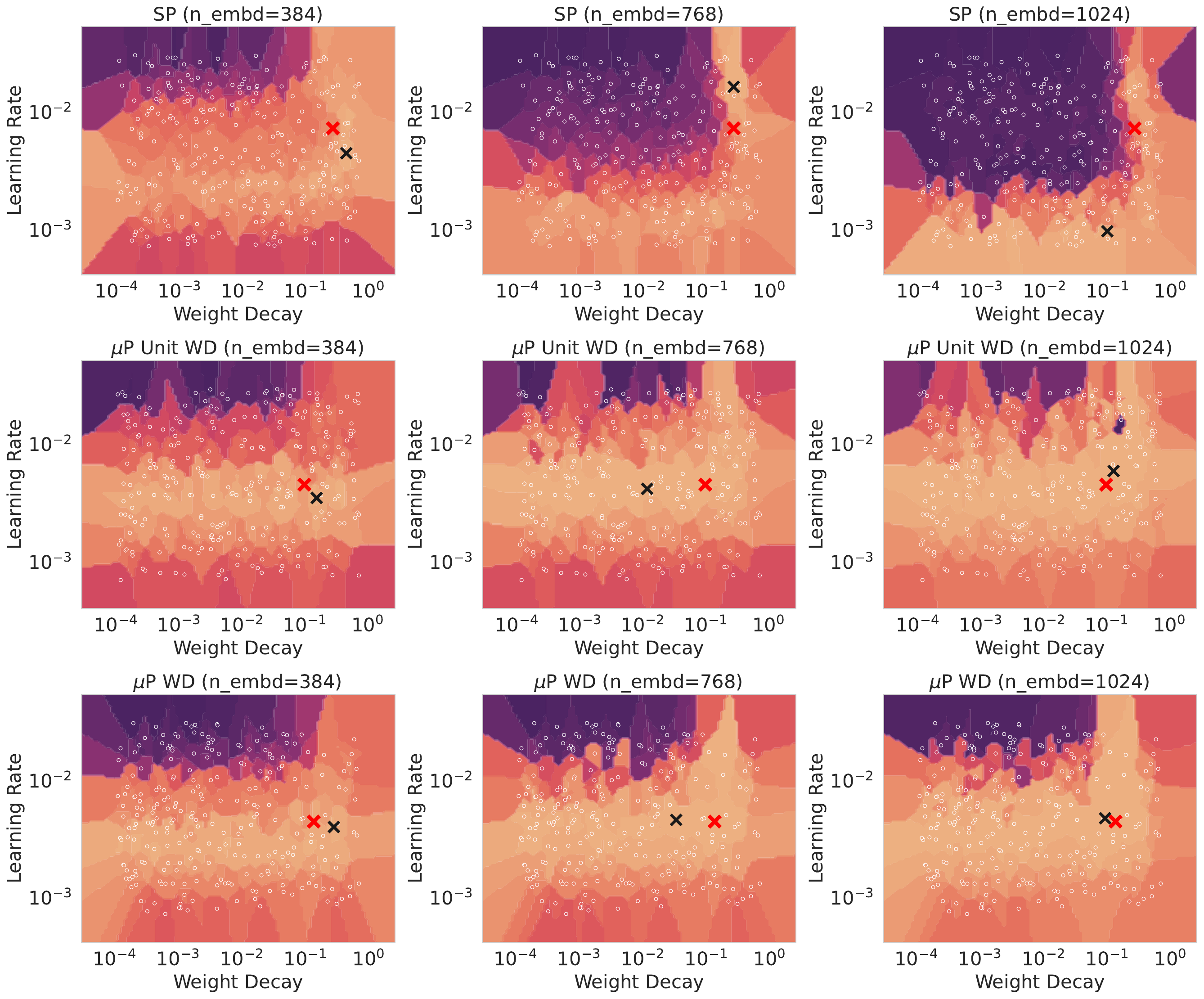}
    \caption{Voronoi interpolation for random sweeps over both learning rate and weight decay.
    The top row is standard parameterization. The middle row is the vanilla Adam-\mup implementation suggested in \cite{yangtensorv}.
    The bottom row is our proposed implementation. Each column corresponds to a different size model, with the number of parameters increasing from left to right.
    For each model and implementation, we plot the best trial.
    Hidden dimension, depth, batch size, and training iterations are all scaled.
    Lighter colors indicate lower loss, darker colors indicate higher loss.
    The red crosses mark the average \texttt{(learning rate, weight decay)} pair, where each coordinate is averaged over the model sizes, while the black crosses are the optimal pair for each experiment.}
    \label{fig:voronoi_wd}
\end{figure*}

\begin{figure}[!t]
    \centering
    \includegraphics[width=\linewidth]{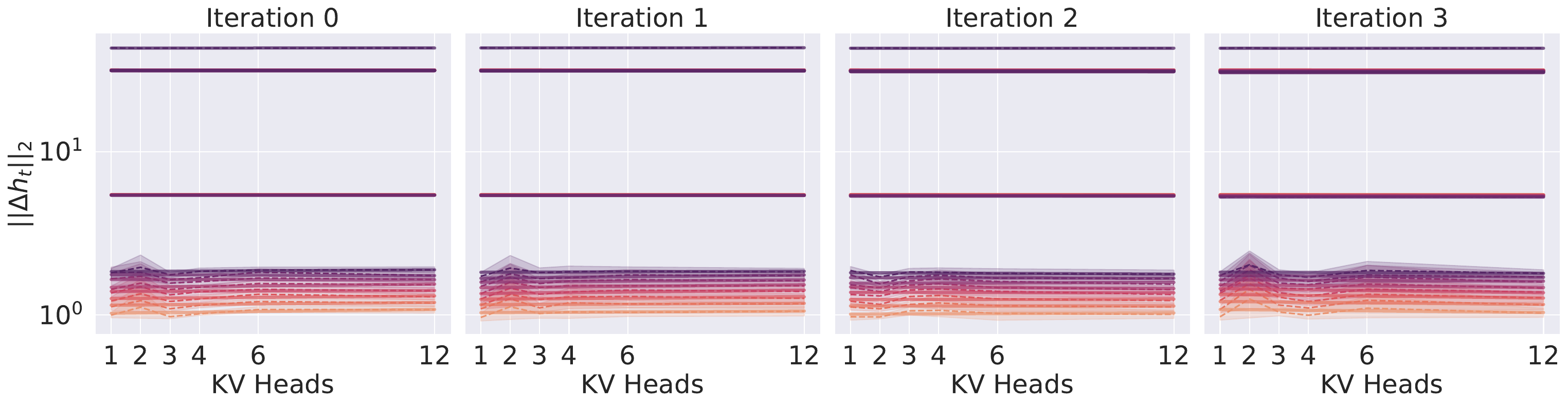}
    \caption{Coordinate checks in the style of \citet{yangtensorv} for  the activation update norms $\norm{\Delta h^\ell_t}_2$ under the vanilla Adam-\mup implementation.
    Together, these coordinate checks indicate that the implementation is correct, and that feature learning and thus learning rate transfer should occur.
    However, Figure~\ref{fig:gqa_ablation_no_fsdp} shows that learning rate transfer for this implementation does not occur. 
    }
    \label{fig:df_input_norm}
\end{figure}

\begin{figure*}[!t]
    \centering
    \includegraphics[width=\linewidth]{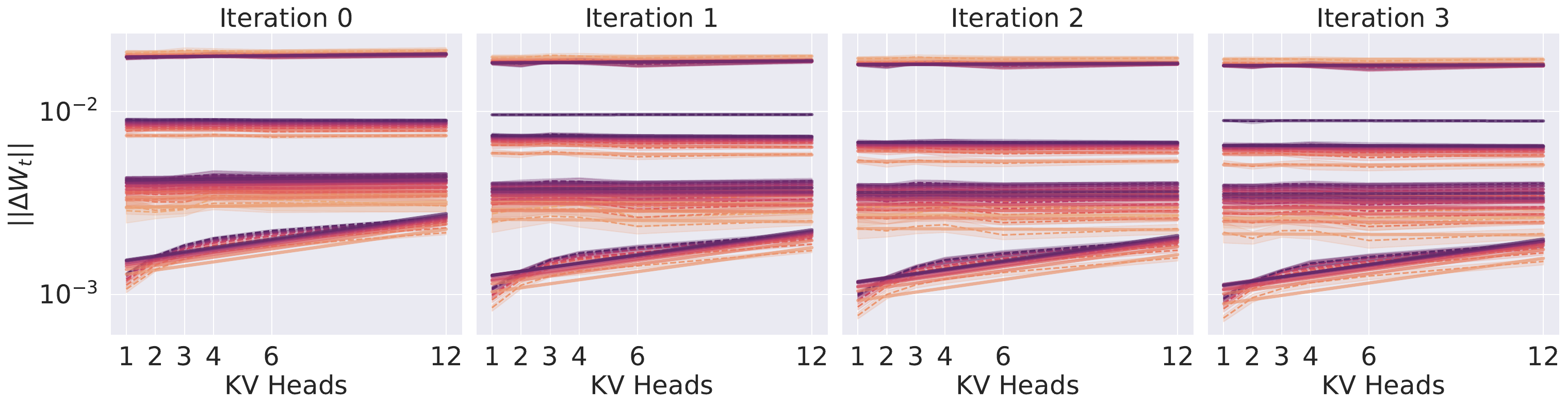}
    \caption{Coordinate checks for $\norm{\Delta \bm{W}}$ under the vanilla Adam-\mup scalings.
    The model fails the coordinate checks when evaluated using the spectral feature learning condition~\eqref{eq:spectral_condition}.
    However, as shown in Figure~\ref{fig:df_input_norm}, it does pass when evaluated under Yang’s definition of feature learning~\ref{eq:yang_feature}.
    }
    \label{fig:dw_adam}
\end{figure*}

\begin{figure*}[!t]
    \centering
    \includegraphics[width=\linewidth]{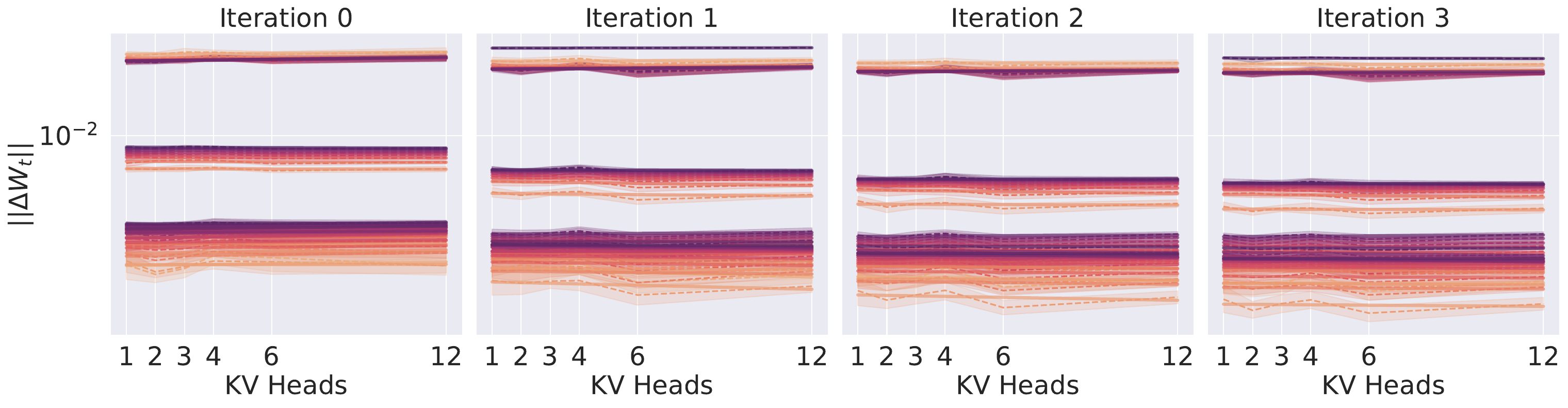}
    \caption{Coordinate checks for $\norm{\Delta \bm{W}}$ under our proposed GQA scalings.
    The model has eight hidden layers. Additional experimental details are provided in Appendix~\ref{app:cc_exp}.}
    \label{fig:dw_gqa}
\end{figure*}

\paragraph{Coordinate-Checks Demonstrate the Necessity for Spectral Feature Learning:}
As discussed in Section~\ref{sec:derivation}, validating feature learning by measuring the norms of $h^\ell$ and $\Delta h_t^\ell$ can be misleading.
Figures~\ref{fig:df_input_norm} plot $\norm{\Delta h^\ell_t}$ for the vanilla Adam-\mup implementation.
The coordinate check would suggest transferable learning rates, yet empirical results show otherwise (see Figure~\ref{fig:gqa_ablation_no_fsdp}, middle).
By contrast, when we instead examine the spectral norm conditions in~\eqref{eq:spectral_condition} (Figure~\ref{fig:dw_adam}), the model fails the coordinate check: a clear non-linear dependence on the number of KV heads of the model, which explains the lack of transferable dynamics.

\paragraph{Coordinate-Checks Demonstrate a Qualitative Dependency on $r$:}
Because the vanilla Adam-\mup implementation and our implementation share the same initialization scaling, we do not compare $\norm{\bm{W}}$ directly.
Instead, Figure~\ref{fig:dw_adam} presents the coordinate checks for the vanilla Adam-\mup implementation, while Figure~\ref{fig:dw_gqa} shows the corresponding coordinate checks for our proposed GQA scaling.
Our method passes the coordinate check, thereby enabling $\mu$-transfer of learning rate.
By contrast, the vanilla Adam-\mup implementation shows a persistent dependency on the number of KV heads, explaining why the learning rate does not transfer in this case.

\paragraph{Learning Rate Transfer for Grouped Query Attention:}
We perform an ablation study comparing the standard parameterization, the vanilla Adam-\mup implementation (where the KV heads are initialized as hidden layers), and our proposed GQA-\mup.
The results of this ablation study are summarized in Figure~\ref{fig:gqa_ablation_no_fsdp}.
We observe that the vanilla Adam-\mup scaling does not account for the shift induced by using GQA, whereas our proposed scaling brings the optimal learning rates into a much narrower region.
Noise inherent to GQA training is already evident in these plots and becomes more pronounced as the number of KV heads decreases.
This noise is apparent in both the coordinate checks from Figure~\ref{fig:dw_gqa} as well as in Figure~\ref{fig:expected_norms}.
We provide an explanation for this phenomenon in the following paragraph.

\paragraph{Expected Variance in GQA Transfer:}
From the perspective of \mup, the nature of GQA introduces a dichotomy: one may achieve feature learning in the sense of~\eqref{eq:spectral_condition} and thereby obtain learning rate transfer, but at the cost of increasingly noisy dynamics as the number of KV heads decreases;
alternatively, one may constrain the variance as we decrease the number of KV heads to stabilize the training, but this leads to a shift in optimal learning rate.
Consequently, we suggest that in scenarios where transferable dynamics are critical, it may be preferable to avoid using GQA altogether.

\paragraph{\mup (Mostly) Decouples Coupled Weight Decay}
To examine the transferability of optimal learning rate and optimal weight decay across model scales, we do a random grid search over (learning rate, weight decay) pairs at constant initial standard deviation. We plot our results in Figure~\ref{fig:voronoi_wd}. 
We note that under the standard parameterization, neither the learning rate nor weight decay transfers, and that the qualitative properties of the Voronoi-interpolated loss landscape change markedly as the model size increases from 26M to 177M non-embedding parameters.
By contrast, both the vanilla Adam-\mup implementation and our proposed scaling preserve their qualitative properties across model sizes.

For the experiment in Figure~\ref{fig:voronoi_wd}, we quantify the degree of transfer in Table~\ref{tab:voronoi_batch}.
We find that the variance of both the optimal learning rate and the optimal weight decay across model sizes is lower for our implementation than for the vanilla Adam-\mup baseline.
Thus, it suggests that our proposed implementation enables the transfer of both learning rate and weight decay across model scales, both qualitatively and quantitatively.

\begin{wraptable}[10]{r}{0.45\linewidth}
\centering
\caption{Variance table comparing our implementations across model sizes for the $\tau_{\text{epoch}}$ experiment from Figure~\ref{fig:voronoi_tau}.}
\label{tab:voronoi_tau}
\resizebox{\linewidth}{!}{%
\begin{tabular}{llll}
\toprule
Implementation           & Var. LR              & Var. $\tau_{\text{epoch}}$ & Var. Loss            \\ 
\midrule
SP                       & $1.34$               & $2.78$                     & $4.87\times 10^{-1}$ \\ 
$\mu$P                   & $4.75\times 10^{-2}$ & $1.49$                     & $4.87\times 10^{-1}$ \\ 
$\mu$P + WD              & $5.54\times 10^{-3}$ & $6.56\times 10^{-1}$       & $4.77\times 10^{-1}$ \\ 
\bottomrule
\end{tabular}}%
\end{wraptable}

Previous works have argued that the quantity $\tau_{\text{epoch}}=(\lambda_0 \times \eta_0 \times \text{iters})^{-1}$ should transfer instead of weight decay \citep{wang2024set, bergsma2025power}.
We found that both weight decay and $\tau_{\text{epoch}}$ transfer in our experimental setting.
This is a non-trivial observation since we vary the number of iterations based on the model size.
Figure~\ref{fig:voronoi_tau} presents the analog of our interpolation diagram.
Figure~\ref{fig:voronoi_wd} and Table~\ref{tab:voronoi_tau} reports the quantitative variance results for $\tau_{\text{epoch}}$.
We find that $\tau_{\text{epoch}}$ transfers slightly better than weight decay in our setting.
\section{Conclusions}
In this paper, we introduced a novel extension of the spectral \mup framework originally developed by \cite{yang2023spectral}. We can apply our framework to rederive the Complete-P weight decay and depth scalings from \cite{dey2025don}. Additionally, we use our framework to derive, for the first time, the \mup scalings for grouped query attention \citep{ainslie2023gqa}. We perform empirical validation in two directions for our work. First, we explore the empirical nature of learning rate transfer for GQA. We find that we can either do noisy learning rate transfer or fail to transfer the learning rate. This dichotomy is a consequence of the competing scalings between the spectral norm and the expected operator norm. Compared to the standard \mup implementation, our method reduces the variance in optimal learning rate during learning rate transfer. Second, we explore the transferability of weight decay across model sizes. We demonstrate that with the standard \mup implementation, we can nearly achieve transfer of weight decay. With our implementation, we are able to get much closer to true transfer across both learning rate and weight decay. Future work could extend scaling laws for Mixture-of-Experts or other sparse models.


\bibliography{neurips_2026}
\bibliographystyle{iclr2026_conference}

\newpage
\appendix

\section{Additional Mathematical Details}\label{eq:detailed_math}
\subsection{Derivation for Adam}\label{app:adam}
We demonstrate the applicability of our framework by re-deriving the \mup scalings for Adam. Recall that the Adam optimizer~\cite{kingma2014adam} uses hyperparameters $\beta_1$, $\beta_2$, $\ep$, and $\eta$ and has its optimization steps given by the following components:
\begin{align*}
    g_t &= \nabla_{\bm{W}}f(\bm{W}_{t-1}), \\
    m_t &= \beta_1m_{t-1} + (1-\beta_1)g_t, \qquad
    v_t = \beta_2v_{t-1} + (1-\beta_2)g_t^2, \\
    \hat{m}_t &= \frac{m_t}{1 - \beta_1^t}, \qquad\qquad\qquad\qquad
    \hat{v}_t = \frac{v_t}{1 - \beta_2^t},
\end{align*}
with the weight update
\begin{align}\label{eq:adam_update}
    \bm{W}_t  &= \bm{W}_{t-1}-\eta\frac{\hat{m}_t}{\sqrt{\hat{v}_t}+\ep}.
\end{align}
The key observation is that the term
\begin{align}\label{eq:adam_r}
    \hat{\bm{r}}_t := \frac{\hat{m}_t}{\sqrt{\hat{v}_t}+\ep}
\end{align}
will always have typical size $1$ (for $\ep$ sufficiently small), and as such the spectral norm can be estimated using the Bai-Yin theorem \citep{yin1988limit,bai1993limit}, depending on whether the layer is vector-like or matrix-like. Thus, we have the following reasoning. For an input layer, we have
\begin{align*}
    \norm{\Delta \bm{W}^0_{t}} = \underbrace{\Theta(\eta^0 \sqrt{n})}_{\text{Bai-Yin}} = \underbrace{\Theta(\sqrt{n})}_{\text{\eqref{eq:spectral_condition}}},
\end{align*}
which implies that we must choose $\eta^0=\Theta(1)$. Next, for the hidden layers, we have
\begin{align*}
    \norm{\Delta \bm{W}^\ell_{t}} = \underbrace{\Theta(\eta^\ell n)}_{\text{Bai-Yin}} = \underbrace{\Theta(1)}_{\text{\eqref{eq:spectral_condition}}},
\end{align*}
which leads us to choose $\eta^\ell = \Theta(n^{-1})$. Finally, for the output layer, we have
\begin{align*}
    \norm{\Delta \bm{W}^{L+1}_{t}} = \underbrace{\Theta(\eta^{L+1} \sqrt{n})}_{\text{Bai-Yin}} = \underbrace{\Theta(n^{-1/2})}_{\text{\eqref{eq:spectral_condition}}},
\end{align*}
which leads us to choose $\eta^{L+1}=\Theta(n^{-1})$. There is a subtle nuance in our derivation that is also often overlooked in the literature. We have assumed that the Adam optimizer step is independent of the network width $n$, but this is not quite true. To see why, consider setting $\beta_1=\beta_2=1$, so that the Adam optimizer step is given simply by
\begin{align*}
    \hat{\bm{r}}_t = \frac{g_t}{|g_t|+\ep},
\end{align*}
where $g_t$ is the gradient. For concreteness, consider a hidden layer. \cite{yangtensorv} show that the gradient will scale like $\Theta(1/n)$. Thus letting $\overline{g}_t=g_t/n$ be the size $1$ normalized gradient updates, we have that the
\begin{align*}
    \hat{\bm{r}}_t = \frac{\overline{g}_t}{|\overline{g}_t|+n\ep},
\end{align*}
which is not actually $\Theta(1)$ in $n$, since for $n = \Omega(\ep^{-1})$, the Adam updates decay like $n^{-1}$. Thus, to be pedantic and ensure actual feature learning, we must scale $\ep = \ep_0 / n$ to continue to achieve feature learning. In practice, we find that this subtlety can be avoided by setting $\ep = 10^{-12}$ instead of the usual default of $10^{-8}$; however, for a complete treatment, this scaling must be included.  We note that \cite{dey2023cerebras}, \cite{everett2024scaling} make the same conclusion about scaling the Adam $\ep$ parameter and perform an empirical study on its transferability.

\subsection{Additional Details for Grouped Query Attention}
\begin{lemma}{Scaling Concatenated Spectral Norms.}\label{lem:concat}
    Let $A \in \bm{R}^{n \times \frac{n}{r}}$ for some integer $r \geq 1$, where $r$ is the number of repetitions of key and value heads, be a matrix with spectral norm $\norm{A} > 0$. Then letting $A^{\oplus} := \bigoplus_{j=1}^r A$ we have
    \begin{align}
        \norm{A^{\oplus}} = \sqrt{r}\norm{A}.
    \end{align}
\end{lemma}

\begin{proof}[\textbf{Proof of Lemma~\ref{lem:concat}}.]
We denote the $r$-times concatenation by 
\begin{align}\label{eq:stacked_A}
    A^* = [\,\underbrace{A\,A\,\cdots\,A}_{\text{$r$ times}}\,].
\end{align}
Each matrix $A$ has a singular value decomposition $U\Sigma V^T$ for $U, V$ unitary with $U \in \R^{n \times n}$, $V \in \R^{\frac{n}{r}\times \frac{n}{r}}$, and $\Sigma = \begin{bmatrix}\Lambda \\ 0\end{bmatrix} \in \R^{n \times \frac{n}{r}}$ with $\Lambda$ a diagonal $\R^{\frac{n}{r}\times \frac{n}{r}}$ matrix. Substituting the SVD into~\eqref{eq:stacked_A} we can factor out the unitary matrix $U$ and arrive at
\begin{align*}
    A = U\,[\,\Sigma V^T\,\Sigma V^T\,\cdots\,\Sigma V^T\,]
    = U\,\begin{bmatrix}
        \Lambda V^T & \Lambda V^T & \cdots & \Lambda V^T \\
        0 & 0 & \cdots & 0 \\
        \vdots & \vdots & \ddots & \vdots \\
        0 & 0 & \cdots & 0
    \end{bmatrix}
\end{align*}
It remains to find the singular values of this matrix, which give us the spectral norm scaling. To this end, observe that by the unitary of $V$ we have
\[
\begin{aligned}
& U
\begin{bmatrix}
    \Lambda V^T & \Lambda V^T & \cdots & \Lambda V^T \\
    0           & 0           & \cdots & 0 \\
    \vdots      & \vdots      & \ddots & \vdots \\
    0           & 0           & \cdots & 0
\end{bmatrix}
\begin{bmatrix}
    V \Lambda & 0      & \cdots & 0 \\
    V \Lambda & 0      & \cdots & 0 \\
    \vdots    & \vdots & \ddots & \vdots \\
    V \Lambda & 0      & \cdots & 0
\end{bmatrix}
U^T
= U
\begin{bmatrix}
    r\Lambda^2 & 0      & \cdots & 0 \\
    0          & 0      & \cdots & 0 \\
    \vdots     & \vdots & \ddots & \vdots \\
    0          & 0      & \cdots & 0
\end{bmatrix}
U^T
\end{aligned}
\]

Thus, the largest eigenvalue of $AA^T$ is given by $r\lambda_{\text{max}}^2$, with $\lambda_{\text{max}}$ being the largest eigenvalue of $A$, and the desired spectral norm scaling is immediate.
\end{proof}

\begin{lemma}\label{lem:asymptotc_equivalence_of_norms}
    Let $A \in \bm{R}^{n \times n}$ have i.i.d. entries. Then the for $x\sim \mathcal{N}(0, 1)$ with i.i.d. entries, we have that
    \begin{align*}
        \norm{\bm{A}}_{\E} = \Theta(\norm{\bm{A}}).
    \end{align*}
\end{lemma}

\begin{proof}[\textbf{Proof of Lemma~\ref{lem:asymptotc_equivalence_of_norms}}]
    First, note that $\norm{\bm{A}}=\Theta(\sigma\sqrt{n})$, where $\sigma$ is the variance of the i.i.d. entries of $\bm{A}$. Next, observe that we can use the law of large numbers and the central limit theorem to estimate
    \begin{align*}
        \E\norm{Ax}_2^2 = \E\sum_{i}\left(\sum_j A_{ij}x_j\right)^2 = \Theta(\sigma^2n^2),
    \end{align*}
    and the result follows since $\norm{x}_2 = \Theta(\sqrt{n})$.
\end{proof}

\subsection{Intuition behind no exact cancellation}
\label{sec:intuition_cancellation}
The intuition is that, in high-dimensional spaces, independently initialized weight matrices and their gradient updates typically act on inputs in weakly correlated, nearly orthogonal directions. Therefore, when we add two such high-dimensional operators, such as $W_{t-1}$ and $\Delta W_{t-1}$, their geometries are unlikely to align in a way that causes significant cancellation. As a result, the norm of the sum remains on the same order as the combined contribution of the two terms, rather than becoming artificially small.

 This assumption captures a basic stability property of high-dimensional neural networks: when an update is added to a weight matrix, the update and the existing weights should not systematically point in opposite directions. If they did, the update could cancel the weights, causing the activations or backpropagated gradients to shrink across layers or training steps. In the extreme case, repeated cancellation would drive internal signals toward zero, preventing the network from learning useful features from the data.
\section{Experimental Details}
\subsection{Model Configurations}\label{app:model_configs}

In our experiments, we train Transformer language models with untied embeddings and GELU \cite{hendrycks2023gaussianerrorlinearunits} nonlinearity.
The batch size is chosen using a data-driven optimal batch size in~\eqref{eq:batch_tokens} based on the total number of training tokens $n_{tokens}$, where the corresponding sequence length is 8192.

\begin{align} \label{eq:batch_tokens}
    B = 0.000733 \times \sqrt{n_{tokens}}.
\end{align}

Equation~\ref{eq:batch_tokens} follows the isoloss sweep methodology of \cite{bergsma2026power} but uses a rounded exponent for tractability. Specifically, \cite{bergsma2026power} estimates a scaling exponent of 0.46 and recommends rounding to 0.5. Since we ran independent sweeps on our own data, equation~\ref{eq:batch_tokens} is specific to our setup but aligns structurally with \cite{bergsma2026power}.

We use a cosine learning rate schedule with warmup. 
The number of warmup steps follows~\eqref{eq:warmup}, \cite{dey2025don}:

\begin{align} \label{eq:warmup}
    n_{\text{warmup}} = \min(\text{int}(0.02 * n_{\text{training}}), \text{int}(375e6/(B \times L))),
\end{align}
where B is batch size and L is sequence length.

All of our experiments are conducted using the openwebtext dataset \citep{gokaslan2019OpenWeb}.

\begin{table}[t]
\centering
\caption{Model configurations for the coordinate check experiments from Figures~\ref{fig:df_input_norm}, \ref{fig:dw_adam}, \ref{fig:dw_gqa}.}
\label{tab:cc_exp}
\begin{adjustbox}{width=0.54\textwidth}
\begin{tabular}{cccccc}
\toprule
Width & Depth & \begin{tabular}[c]{@{}c@{}}Num\\ Heads\end{tabular} & Head Size & \begin{tabular}[c]{@{}c@{}}KV \\ Heads\end{tabular} & \begin{tabular}[c]{@{}c@{}}KV \\ Reps\end{tabular} \\
\midrule
576   & 8     & 12                                                  & 64        & 1                                                   & 12                                                 \\ 
576   & 8     & 12                                                  & 64        & 2                                                   & 6                                                  \\ 
576   & 8     & 12                                                  & 64        & 3                                                   & 4                                                  \\ 
576   & 8     & 12                                                  & 64        & 4                                                   & 3                                                  \\ 
576   & 8     & 12                                                  & 64        & 6                                                   & 2                                                  \\ 
576   & 8     & 12                                                  & 64        & 12                                                  & 1   
\\ \bottomrule
\end{tabular}
\end{adjustbox}
\end{table}

\subsubsection{Coordinate Checks}\label{app:cc_exp}

For the coordinate checking we verify that our norms remain stable as we vary the nubmer of kv heads. The specific configurations which we used during the coordinate checks are contained in Table~\ref{tab:cc_exp}. We use weight decay $0$ in our coordinate checking experiments, and do all of the computation in \texttt{float32}. We used a fixed Adam $\ep$ of $10^{-12}$ and an initial standard deviation of $0.02$. Other optimizer settings are set to the defaults of \texttt{PyTorch}'s Adam implementation. We perform our experiments on seeds $1$ through $10$ and plot the average and confidence interval. We use a batch size of $1$ and a sequence length of $1024$ to ensure quick computation.

\subsubsection{GQA Ablation Experiment}\label{app:gqa_exp}

We train our GQA ablation models to 10 TPP. The configurations used for this experiment can be found in Table~\ref{tab:gqa_exp}. We set the base weight decay to be $\lambda_0 = 0.1$. We use a base Adam $\ep$ of $10^{-9} / n$, where $n$ is the embedding dimension, to match the predicted Adam $\ep$ scaling of \cite{dey2025don}. We take three runs for each data point, using seeds $42, 43, 44$ for reproducibility.

\begin{table*}[th]
\centering
\caption{Model configurations for the GQA transfer experiments from Figure~\ref{fig:gqa_ablation_no_fsdp}.}
\label{tab:gqa_exp}
\begin{adjustbox}{width=\linewidth}
\begin{tabular}{lrrrrrrrrrrrrrr}
\toprule
           & \multicolumn{1}{c}{Params} & \multicolumn{1}{c}{\begin{tabular}[c]{@{}c@{}}Non-Embd\\  Params\end{tabular}} & \multicolumn{1}{c}{Width} & \multicolumn{1}{c}{Depth} & \multicolumn{1}{c}{\begin{tabular}[c]{@{}c@{}}Num\\ Heads\end{tabular}} & \multicolumn{1}{c}{Head Size} & \multicolumn{1}{c}{\begin{tabular}[c]{@{}c@{}}KV \\ Heads\end{tabular}} & \multicolumn{1}{c}{\begin{tabular}[c]{@{}c@{}}KV \\ Reps\end{tabular}} & \multicolumn{1}{c}{TPP} & \multicolumn{1}{c}{\begin{tabular}[c]{@{}c@{}}Dataset Size \\ (Tokens)\end{tabular}} & \multicolumn{1}{c}{\begin{tabular}[c]{@{}c@{}}Dataset Size\\  (Sequences)\end{tabular}} & \multicolumn{1}{c}{\begin{tabular}[c]{@{}c@{}}Batch Size \\ (Tokens)\end{tabular}} & \multicolumn{1}{c}{\begin{tabular}[c]{@{}c@{}}Batch Size \\ (Sequences)\end{tabular}} & \multicolumn{1}{c}{Iterations} \\ \midrule
kvr\_t\_1  & 125.55                      & 80.62                                                                           & 768                        & 7                          & 12                                                                       & 64                             & 1                                                                        & 12                                                                      & 10                       & 806200000                                                                             & 98413                                                                                    & 262144                                                                              & 32                                                                                     & 3075                            \\ 
kvr\_t\_2  & 126.23                      & 81.31                                                                           & 768                        & 7                          & 12                                                                       & 64                             & 2                                                                        & 6                                                                       & 10                       & 813100000                                                                             & 99255                                                                                    & 262144                                                                              & 32                                                                                     & 3102                            \\ 
kvr\_t\_3  & 126.92                      & 82                                                                              & 768                        & 7                          & 12                                                                       & 64                             & 3                                                                        & 4                                                                       & 10                       & 820000000                                                                             & 100098                                                                                   & 262144                                                                              & 32                                                                                     & 3128                            \\ 
kvr\_t\_4  & 127.61                      & 82.69                                                                           & 768                        & 7                          & 12                                                                       & 64                             & 4                                                                        & 3                                                                       & 10                       & 826900000                                                                             & 100940                                                                                   & 262144                                                                              & 32                                                                                     & 3154                            \\ 
kvr\_t\_6  & 128.99                      & 84.06                                                                           & 768                        & 7                          & 12                                                                       & 64                             & 6                                                                        & 2                                                                       & 10                       & 840600000                                                                             & 102612                                                                                   & 262144                                                                              & 32                                                                                     & 3207                            \\ 
kvr\_t\_12 & 133.12                      & 88.19                                                                           & 768                        & 7                          & 12                                                                       & 64                             & 12                                                                       & 1                                                                       & 10                       & 881900000                                                                             & 107654                                                                                   & 270336                                                                              & 33                                                                                     & 3262                            \\ \bottomrule
\end{tabular}
\end{adjustbox}
\end{table*}

\subsubsection{Weight Decay Transfer Experiment}\label{app:wd_exp}

Due to the high number of sampling points we only trained our models in the weight decay experiments to 3 TPP, well below the compute optimal horizon \citep{hoffmann2022empirical}. The configurations used for this experiment can be found in Table~\ref{tab:wd_exp}. We uniformly sample the grid in $\log-\log$ space. We were only able to run one trial per data point, but the high number of trials increase confidence. We sample 250 points on the grid for each implementation and model size.

\begin{table*}[th]
\centering
\caption{Model configurations for the Weight decay experiments from Figures~\ref{fig:voronoi_wd} and~\ref{fig:voronoi_tau}.}
\label{tab:wd_exp}
\begin{adjustbox}{width=\linewidth}
\begin{tabular}{ccccccccccccccc}
\toprule
           & Params & \begin{tabular}[c]{@{}c@{}}Non-Embd \\ Params\end{tabular} & Width & Depth & \begin{tabular}[c]{@{}c@{}}Num\\ Heads\end{tabular} & \begin{tabular}[c]{@{}c@{}}Head \\ Size\end{tabular} & \begin{tabular}[c]{@{}c@{}}KV \\ Heads\end{tabular} & \begin{tabular}[c]{@{}c@{}}KV \\ Reps\end{tabular} & TPP & \begin{tabular}[c]{@{}c@{}}Dataset Size \\ (Tokens)\end{tabular} & \begin{tabular}[c]{@{}c@{}}Dataset Size\\  (Sequences)\end{tabular} & \begin{tabular}[c]{@{}c@{}}Batch Size \\ (Tokens)\end{tabular} & \begin{tabular}[c]{@{}c@{}}Batch Size \\ (Sequences)\end{tabular} & Iters. \\ \midrule
jwd-small  & 48.82  & 26.38                                                      & 384   & 4     & 6                                                   & 64                                                   & 6                                                   & 1                                                  & 3   & 79140000                                                         & 9661                                                                & 81920                                                          & 10                                                                & 966    \\ 
jwd-medium & 125.96 & 81.07                                                      & 768   & 6     & 12                                                  & 64                                                   & 12                                                  & 1                                                  & 3   & 243210000                                                        & 29689                                                               & 147456                                                         & 18                                                                & 1649   \\ 
jwd-large  & 237.17 & 177.31                                                     & 1024  & 10    & 16                                                  & 64                                                   & 16                                                  & 1                                                  & 3   & 531930000                                                        & 64933                                                               & 212992                                                         & 26                                                                & 2497   \\ \bottomrule
\end{tabular}
\end{adjustbox}
\end{table*}

\begin{figure*}[!t]
    \centering
    \includegraphics[width=0.84\linewidth]{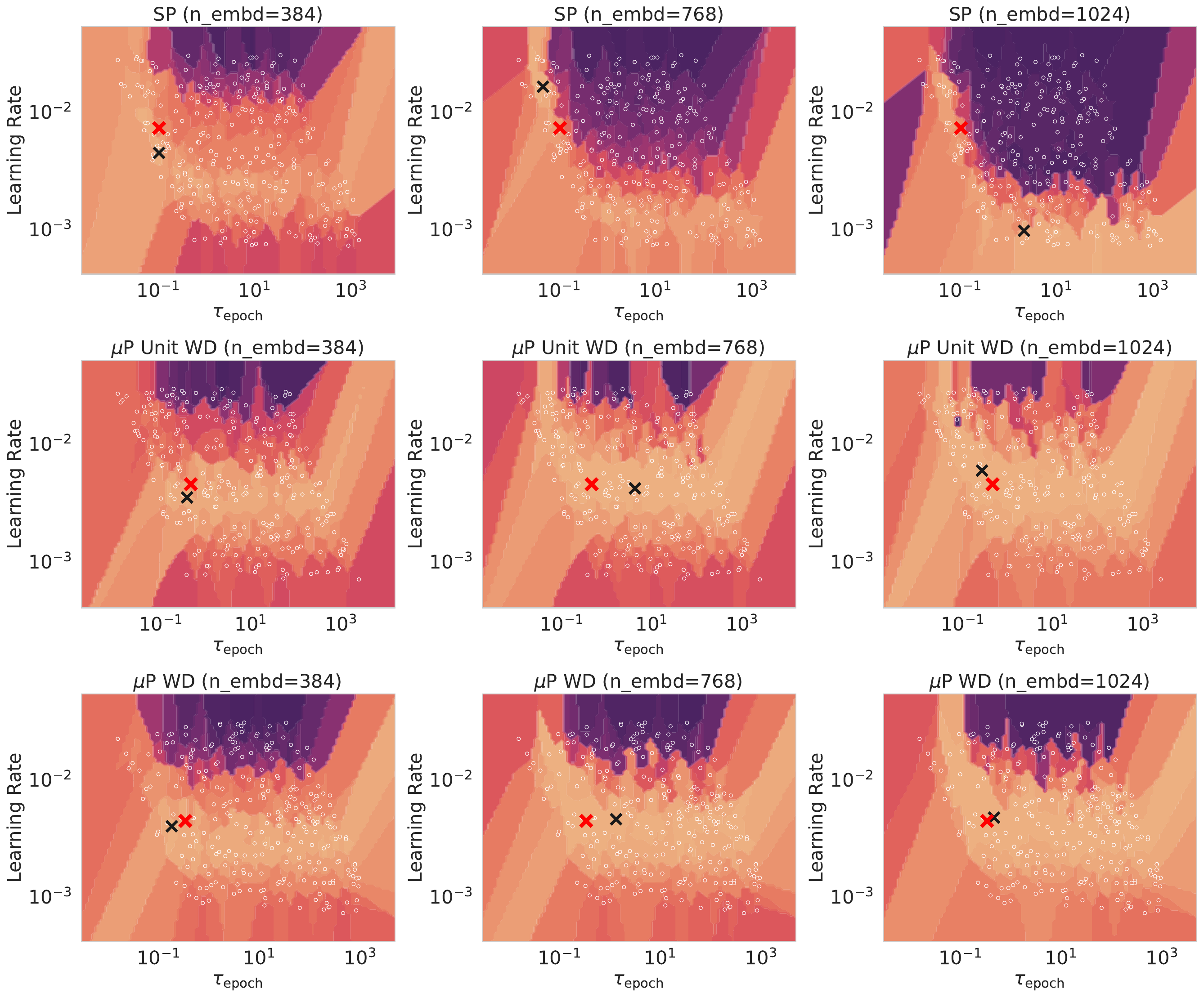}
    \caption{Voronoi interpolation for random sweeps over both learning rate and $\tau_{\text{epoch}}$ \cite{wang2024set}. The top row is standard parameterization. The middle row is the vanilla Adam-\mup implementation suggested in \cite{yangtensorv}. The bottom row is our proposed implementation. Each column is a different size model, increasing in number of parameters from left to right. For each model and implementation we plot the best trial. We scale the hidden dimension, depth, batch size, and training iterations.
    Lighter colors are lower loss, darker colors are higher loss. The red x is the average \texttt{(learning rate, weight decay)} pair, where each coordinate is averaged over the model sizes, while the black x is the optimal pair for each experiment.}
    \label{fig:voronoi_tau}
\end{figure*}

\begin{figure*}[!t]
    \centering
    \includegraphics[width=0.84\linewidth]{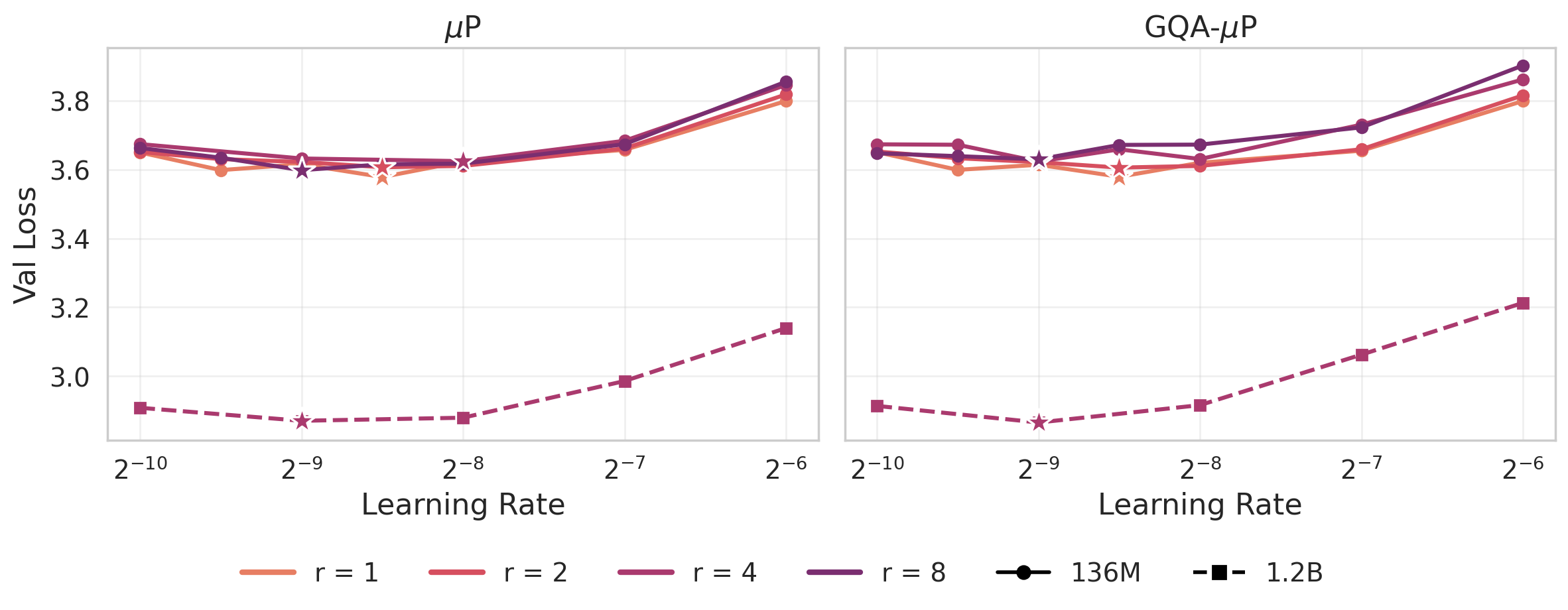}
    \caption{Learning-rate transfer at 20 tokens-per-parameter (TPP) under vanilla Adam-\mup (left) and GQA-\mup (right).
    The 136M proxy is swept at $r\in\{1,2,4,8\}$ on a half-power LR grid and the 1.2B target is swept at $r{=}4$ on whole-power steps.
    Colour encodes $r$ and linestyle encodes model size; stars mark the optimal learning rate for each configuration.
    The 1.2B target optimum lands at $2^{-9}$ under both parameterizations.
    Under vanilla \mup the 136M proxy optima drift to $2^{-8.5}$ for three of the four $r$ values, whereas under GQA-\mup two of the four proxy optima coincide with the target at $2^{-9}$; the residual $r$-dependence is smaller but not fully eliminated at this training horizon.}
    \label{fig:tpp20_lr_transfer}
\end{figure*}

\subsection{Failure of Yang-Type Coordinate Checking}\label{app:yang_coordinate}
\cite{yangtensorv} suggest measuring $\norm{h_t}_2$ and $\norm{\Delta h_t}_2$ to verify that a \mup implementation is correct by comparing these norms during training to feature learning conditions in~\eqref{eq:yang_feature}. In Figure~\ref{fig:df_input_norm} we plot $\norm{\Delta h_t}_2$ for the vanilla Adam-\mup implementation while varying only the number of kv heads. Note that the the implementation passes a coordinate check, but as discussed theoretically in Section~\ref{sec:derivation}, and empirically in Section~\ref{sec:empirical} the learning rate does not transfer for this implementation (see Figure~\ref{fig:gqa_ablation_no_fsdp}). 

Coordinate checks on our proposed spectral condition from~\eqref{eq:spectral_condition}, however, capture the failure of feature learning (see Figure~\ref{fig:dw_adam}).

\subsection{Transfer at Compute-Optimal Training Horizons}
To check that the behaviour observed at 10~TPP persists at more realistic training lengths, we repeat the learning-rate sweep at 20~TPP, scaling from a 136M proxy model to a 1.2B target model across $r\in\{1,2,4,8\}$.
Figure~\ref{fig:tpp20_lr_transfer} reports the results on a half-power LR grid for the 136M proxy and whole-power steps for the 1.2B target.
The 1.2B target optimum lands at $2^{-9}$ under both parameterizations, so the transfer question reduces to whether the 136M proxy identifies the same learning rate.
Under vanilla \mup, three of the four 136M proxy optima $r\in\{1,2,4\}$ drift to $2^{-8.5}$, half a power above the 1.2B target, while only $r{=}8$ recovers $2^{-9}$.
Under GQA-\mup, two of the four proxy optima $r\in\{4,8\}$ coincide with the target at $2^{-9}$, and the remaining two $r\in\{1,2\}$ sit at $2^{-8.5}$.
GQA-\mup therefore reduces but does not eliminate the residual $r$-dependence at this horizon. Combined with the coordinate-check result in Figure~\ref{fig:dw_gqa}, this still favours GQA-\mup as the theoretically correct choice, particularly at shorter training horizons where the bias under vanilla \mup is larger as shown in Figure~\ref{fig:gqa_ablation_no_fsdp}.

\subsection{More Results about Weight Decay}
\begin{table}[th]
\footnotesize
\centering
\caption{Variance table comparing our implementations across model sizes for the weight decay experiment from Figure~\ref{fig:voronoi_wd}.}
\label{tab:voronoi_batch}
\begin{tabular}{llll}
\toprule
Implementation           & Var. LR              & Var. WD              & Var. Loss            \\ 
\midrule
SP                       & $1.34$               & $3.83\times 10^{-1}$ & $4.87\times 10^{-1}$ \\ 
$\mu$P                   & $4.75\times 10^{-2}$ & $1.38$               & $4.87\times 10^{-1}$ \\ 
$\mu$P + WD              & $5.54\times 10^{-3}$ & $7.51\times 10^{-1}$ & $4.77\times 10^{-1}$ \\ 
\bottomrule
\end{tabular}%
\end{table}

We used the same data that was collected from Figure~\ref{fig:voronoi_wd} to analyze whether or not our experimental testbed demonstrates transfer over $\tau_{\text{epoch}}$, as is suggested by \citep{wang2024set, bergsma2025power, dey2025don}. We find that we get slightly better transfer in $\tau_{\text{epoch}}$ than we do with weight decay, using the same data. We plot the variance in our optimal configurations in Table~\ref{tab:voronoi_tau}. Like for the case of weight decay transfer (see Figure~\ref{fig:voronoi_wd}), we find that our suggested implementation outperforms both the standard parameterization and the vanilla Adam-\mup implementation from \cite{yangtensorv}.

\section{LLM Statement}
We did not use LLMs in a significant way to aid our research during the completion of this work. Our LLM usage did not extend beyond using code assistants like copilot and for polishing the writing in our manuscript.

\end{document}